\documentclass{cstr}

\usepackage{amsmath,amsfonts,amssymb,latexsym,bm}
\usepackage{amsthm}
\usepackage{mathrsfs}
\usepackage[pdftex]{graphicx}
\usepackage{subfig}
\usepackage{url}
\usepackage{cite}
\usepackage{enumerate}

\usepackage{booktabs, multirow}

\newtheorem{theorem}{Theorem}

\usepackage{algorithm, algorithmic}
\renewcommand{\algorithmicrequire}{{\bf Input:}}
\renewcommand{\algorithmicensure}{{\bf Output:}}

\title{
  FedDCL: a federated data collaboration learning as a hybrid-type privacy-preserving framework based on federated learning and data collaboration
}

\author[1,*]{Akira Imakura}
\author[1]{Tetsuya Sakurai}

\affil[1]{University of Tsukuba, 1-1-1 Tennodai, Ibaraki, Tsukuba 305-8573, Japan}

\email{imakura@cs.tsukuba.ac.jp}

\begin{document}
\maketitle

\begin{abstract}
Recently, federated learning has attracted much attention as a privacy-preserving integrated analysis that enables integrated analysis of data held by multiple institutions without sharing raw data.
On the other hand, federated learning requires iterative communication across institutions and has a big challenge for implementation in situations where continuous communication with the outside world is extremely difficult.
In this study, we propose a federated data collaboration learning (FedDCL), which solves such communication issues by combining federated learning with recently proposed non-model share-type federated learning named as data collaboration analysis.
In the proposed FedDCL framework, each user institution independently constructs dimensionality-reduced intermediate representations and shares them with neighboring institutions on intra-group DC servers.
On each intra-group DC server, intermediate representations are transformed to incorporable forms called collaboration representations.
Federated learning is then conducted between intra-group DC servers.
The proposed FedDCL framework does not require iterative communication by user institutions and can be implemented in situations where continuous communication with the outside world is extremely difficult.
The experimental results show that the performance of the proposed FedDCL is comparable to that of existing federated learning.
\end{abstract}

\section{Introduction}
\subsection{Background}
There is a growing demand for integrated analysis of medical data owned by multiple institutions or countries \cite{nepogodiev2020mortality,brat2020international,mariani2016ckd}.
However, sharing the original medical data is difficult because of privacy concerns, and even if it were possible, we would have to pay huge costs.
Therefore, methods to achieve privacy-preserving analysis in which datasets are collaboratively analyzed without sharing the original data are attracting attention.
\par
A typical technology for this topic is a federated learning system \cite{konevcny2016federated,mcmahan2016communication},
Federated learning iteratively updates the integrated model by aggregating information calculated independently on each user institution.
Federated learning enables construction of the integrated model without sharing raw data through iterative model updates that share machine learning models.
On the other hand, federated learning requires iterative communication across institutions and has a big challenge for implementation in situations where continuous communication with the outside world is extremely difficult.
For example, medical data may be stored on a server isolated from external networks, making continuous communication with the outside extremely difficult.
\par
A motivating example of this paper is privacy-preserving medical data analysis in situation where continuous communication with the outside world is extremely difficult for user institutions (Figure~\ref{fig:motivation}).
If the analysis is conducted using only data from a single municipality, the accuracy may be insufficient because of the small sample size, specifically for rare diseases \cite{mascalzoni2014rare}.
Therefore, we consider integrated analysis of medical data owned by multiple institutions or countries.
\par
In this situation, raw data are held in a distributed manner by multiple medical institutions, which also organize into multiple groups based on affiliated hospitals and the municipality or country to which they belong.
Raw data is stored on data servers at each medical institution, and the institutions' servers are not in continuous communication with the outside world.
On the other hand, each institution group would have a server that cannot store the raw data, but can communicate continuously with the outside world.
Each institution's server can perform temporarily secure communication (e.g., data transfer via external storage devices) with the intra-group server.
\begin{figure*}[!t]
\centering
\includegraphics[scale=0.3, bb = 0 50 1134 567]{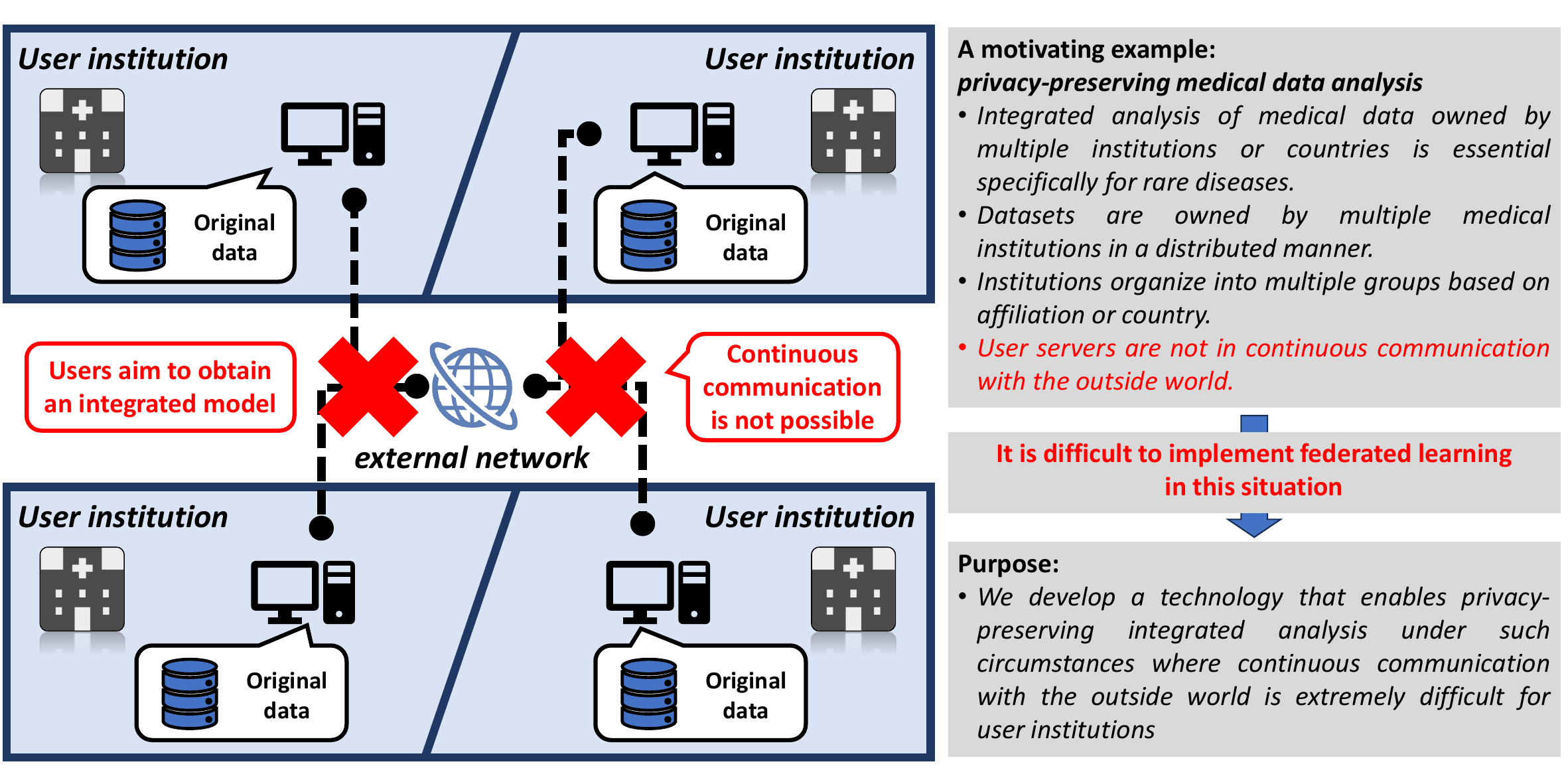}
\caption{
  A motivating example: privacy-preserving medical data analysis in situation where continuous communication with the outside world is extremely difficult for user institutions.
}
\label{fig:motivation}
\end{figure*}
\par
Therefore, developments of technologies for privacy-preserving integrated analysis without any iterative communications by user institutions are essential.
\subsection{Purpose and contributions}
\begin{figure*}[!t]
\centering
\includegraphics[scale=0.3, bb = 0 50 1134 567]{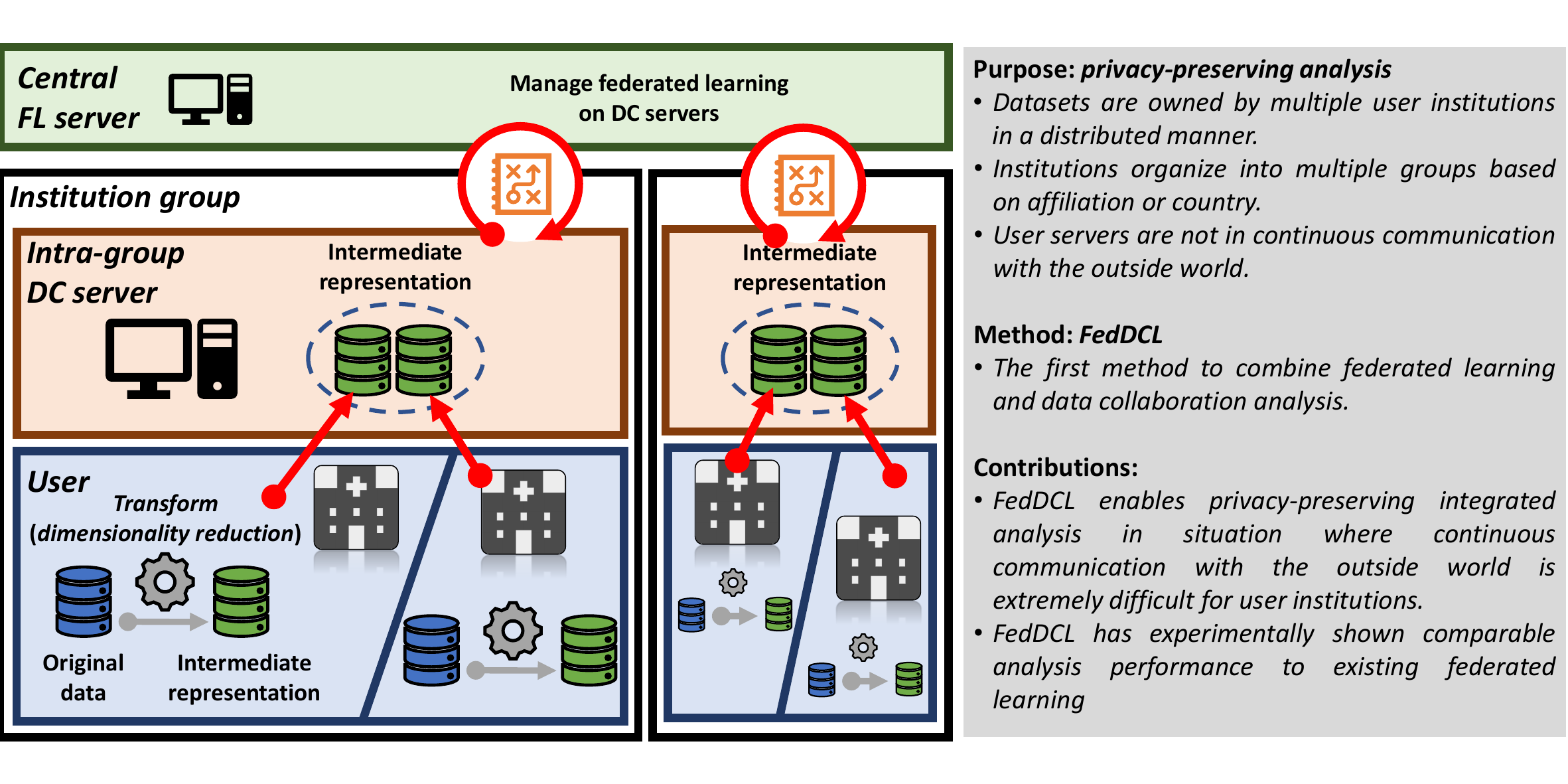}
\caption{
  Concept of the proposed federated data collaboration learning (FedDCL).
}
\label{fig:abstract}
\end{figure*}
The purpose of this paper is to develop a technology that enables privacy-preserving integrated analysis under such circumstances where continuous communication with the outside world is extremely difficult for user institutions.
To tackle such communication issues, we focus on the data collaboration analysis which is a recently proposed non-model share-type federated learning \cite{imakura2020data,imakura2021collaborative,imakura2023non}.
The data collaboration analysis enables integrated analysis without iterative communication across user institutions by not sharing models but dimensionality-reduced intermediate representations.
\par
In this study, we propose a federated data collaboration learning (FedDCL), which solves such communication issues by combining federated learning with the data collaboration analysis; see Figure~\ref{fig:abstract}.
In the proposed FedDCL framework, each user institution independently constructs dimensionality-reduced intermediate representations and shares them with neighboring institutions on intra-group DC servers.
On each intra-group DC server, intermediate representations are transformed to incorporable forms called collaboration representations.
Federated learning is then conducted between intra-group DC servers.
The integrated machine learning model is generated in each user institution combined with mapping functions for constructing intermediate and collaboration representations and federated learning model.
\par
The main contributions of this paper are summarized as
\begin{itemize}
    \item We propose a FedDCL framework that enables privacy-preserving integrated analysis in situation where continuous communication with the outside world is extremely difficult for user institutions.
    \item FedDCL is the first method to combine federated learning and data collaboration analysis.
    \item FedDCL has experimentally shown to have comparable analysis performance to existing federated learning.
    \item FedDCL is a framework that can be easily combined with the latest federated learning mechanisms and deployed on a variety of data and tasks.
\end{itemize}
\section{Related works}
\subsection{Federated learning}
Recently, federated learning systems have been developed for privacy-preserving analysis \cite{konevcny2016federated,mcmahan2016communication}.
Federated learning is typically based on (deep) neural network and updates the model iteratively with iterative communication between user institutions and a central server \cite{li2019survey,konevcny2016federated,mcmahan2016communication,yang2019federated}.
\par
To update the model, federated stochastic gradient descent (FedSGD) and federated averaging (FedAvg) are typical strategies \cite{mcmahan2016communication}. 
FedSGD is a direct extension of the stochastic gradient descent method.
In each iteration of the gradient descent method, each user locally computes a gradient from the shared model using the local dataset and sends the gradient to the server.
The shared gradients are averaged and used to update the model.
\par
Instead of sharing the gradient, we can share model parameters.
This is called FedAvg.
In FedAvg, each user updates the model using the local dataset and sends the updated model to the central server.
Then, the shared models are averaged to update.
FedAvg can reduce the communication frequency than FedSGD.
\par
Federated learning including more recent methods, such as FedProx \cite{li2020federated}, FedCodl \cite{ni2022federated}, FedGroup \cite{duan2021fedgroup}, and FedGK \cite{zhang2024fedgk} and so on, require cross-institutional communication in each iteration called communication round.
Therefore, federated learning has a big challenge to apply in situation where continuous communication with the outside world is extremely difficult for user institutions.
%
\subsection{Data collaboration analysis}
As another approach for privacy-preserving analysis, non-model share-type federated learning called data collaboration analysis has been developed \cite{imakura2020data,imakura2021collaborative,imakura2023non}.
\par
Data collaboration analysis centralizes the dimensionality-reduced intermediate representation to a central server instead of sharing the model.
The intermediate representations are transformed to incorporable forms called collaboration representations.
For constructing the incorporable collaboration representations, all user institutions have a shareable pseudo dataset called anchor dataset and centralize its intermediate representation.
Then, the collaboration representation is analyzed as a single dataset on the central server without communication.
\par
Data collaboration analysis preserves the privacy of the original data by allowing each user to use individual functions to generate the intermediate representation and not share them \cite{imakura2021accuracy}.
The data collaboration analysis does not require iterative communications between user institutions.
\par
Data collaboration analysis has been extended to novelty detection \cite{imakura2021collaborative2}, feature selection \cite{ye2019distributed}, interpretable model construction \cite{imakura2021interpretable}, survival analysis \cite{imakura2023dc-cox}, causal inference \cite{kawamata2024collaborative}, and so on.
\section{Proposal for FedDCL}
This paper targets classification and regression problems on structured data.
That is, for training dataset $X = [{\bm x}_1, {\bm x}_2, \dots, {\bm x}_n]^\top \in \mathbb{R}^{n \times m}$ and $Y = [{\bm y}_1, {\bm y}_2, \dots,$ $ {\bm y}_n]^\top \in \mathbb{R}^{n \times \ell}$, we aim to generate a machine learning model $t$ such that
\begin{equation*}
    t(X) \approx Y.
\end{equation*}
\par
Here, we consider the situation where these $n$ samples of data are held by multiple user institutions in a distributed manner, and where the user institutions are divided into multiple groups.
Let $d$ $(\geq 2)$ be the number of groups and $c_i$ $(\geq 1$) $(i = 1,2,\dots,d)$ be the number of institutions in the $i$-th group.
Here, the total number of institutions is $c = \sum_{i=1}^d c_i$.
Then, the dataset $X$ and $Y$ are distributed into $c$ user institutions as
\begin{equation*}
  X
  = \left[
    \begin{array}{c}
      X^{(1)} \\
      X^{(2)} \\
      \vdots \\
      X^{(d)}
    \end{array}
  \right], \quad
  X^{(i)}
  = \left[
    \begin{array}{c}
      X^{(i)}_1 \\
      X^{(i)}_2 \\
      \vdots \\
      X^{(i)}_{c_i}
    \end{array}
  \right], \quad
  Y
  = \left[
    \begin{array}{c}
      Y^{(1)} \\
      Y^{(2)} \\
      \vdots \\
      Y^{(d)}
    \end{array}
  \right], \quad
  Y^{(i)}
  = \left[
    \begin{array}{c}
      Y^{(i)}_1 \\
      Y^{(i)}_2 \\
      \vdots \\
      Y^{(i)}_{c_i}
    \end{array}
  \right],
\end{equation*}
where $X^{(i)}_j \in \mathbb{R}^{n_{ij} \times m}, Y^{(i)}_j \in \mathbb{R}^{n_{ij} \times \ell}$, and $n = \sum_{i,j} n_{ij}$.
Here, each $(i,j)$-th user institution has a partial dataset $X^{(i)}_j$ and $Y^{(i)}_j$.
\par
All user institutions do not want to share the original data $X^{(i)}_j$, but aim to obtain the model function $t$ trained on dataset from all user institutions.
\subsection{Basic concept}
The basic concept of the proposed FedDCL is shown as follows.
\begin{itemize}
    \item Based on the data collaboration framework, each user institution independently constructs dimensionality-reduced intermediate representations and shares them within each group on an {\it intra-group DC server}.
    At this time, data privacy is ensured by not sharing the mapping function to the intermediate representation.
    \item To enhance data privacy, any data uploaded to the intra-group DC servers from each user is not directly disclosed outside the group.
    \item An integrated analysis model is constructed based on federated learning framework on {\it intra-group DC servers} with a {\it central FL server}.
\end{itemize}
\subsection{Derivation}
Based on the data collaboration and federated learning frameworks, the proposed FedDCL operates by three roles: {\it users}, {\it intra-group DC servers}, and a {\it central FL server}.
FedDCL consists of the following five steps:
Step~1. Construction of shareable pseudo anchor dataset;
Step~2. Construction of intermediate representation;
Step~3. Construction of collaboration representation;
Step~4. Construction of integrated model for collaboration representation;
Step~5. Construction of integrated model for raw dataset.
%
%
\subsubsection*{Step 1: Construction of shareable pseudo anchor dataset}
All users generate the same anchor dataset $A \in \mathbb{R}^{r \times m}$, which is shareable pseudo data consisting of public or dummy data randomly constructed, where $r$ is the number of anchor data samples.
\par
Anchor dataset $A$ is generated by, e.g., uniform random numbers with value ranges for each feature aligned with the raw data.
On the other hand, it is also known that having a data distribution close to that of the raw data improves recognition performance, and a low-rank approximation-based method \cite{imakura2021interpretable} and synthetic minority oversampling technique (SMOTE)-based method \cite{imakura2023another} have been proposed.
\subsubsection*{Step 2: Construction of intermediate representation}
Using a linear or nonlinear row-wise mapping function $f^{(i)}_j$, each $(i,j)$-th user constructs dimensionality-reduced intermediate representations,
\begin{align*}
  \widetilde{X}^{(i)}_j = f^{(i)}_j(X^{(i)}_j) \in \mathbb{R}^{n_{ij} \times \widetilde{m}_{ij}}, \quad
  \widetilde{A}^{(i)}_j = f^{(i)}_j(A) \in \mathbb{R}^{r \times \widetilde{m}_{ij}},
\end{align*}
and centralizes them to the corresponding intra-group DC server.
A typical setting for $f^{(i)}_j$ is dimensionality reduction with $\widetilde{m}_{ij} < m$, including unsupervised \cite{pearson1901liii,maaten2008visualizing} and supervised methods \cite{fisher1936use,sugiyama2007dimensionality,sugiyama2008direct,li2017locality,imakura2019complex}.
\subsubsection*{Step 3: Construction of collaboration representation}
The intermediate representations on the intra-group DC servers cannot be analyzed as one dataset even using federated learning, because $f^{(i)}_j$ depends on users $i$ and $j$.
Therefore, we transform the intermediate representations to incorporable collaboration representation.
\par
If we use a linear transformation, the collaboration representation should be set such that 
\begin{equation*}
  \widetilde{A}^{(i)}_jG^{(i)}_j \approx \widetilde{A}^{(i')}_{j'}G^{(i')}_{j'},
\end{equation*}
where $G^{(i)}_j \in \mathbb{R}^{\widetilde{m}_{ij} \times \widehat{m}}$.
Now, we consider setting the matrix $G^{(i)}_j$ using the following minimal perturbation problem:
\begin{equation*}
  \min_{E^{(i)}_j, G'^{(i)}_j, \|Z'\|_{\rm F} = 1} \sum_{i=1}^{d} \sum_{j=1}^{c_i} \| E^{(i)}_j \|_{\rm F}^2 \quad
  \mbox{s.t. } (\widetilde{A}^{(i)}_j + E^{(i)}_j) G'^{(i)}_j = Z'.
\end{equation*}
This can be solved by a singular value decomposition (SVD) based algorithm for total least squares problems.
Let 
\begin{align*}
   \widetilde{A} 
   &= [\widetilde{A}^{(1)}_1, \dots, \widetilde{A}^{(1)}_{c_1}, \widetilde{A}^{(2)}_{1}, \dots, \widetilde{A}^{(2)}_{c_2}, \dots, \widetilde{A}^{(d)}_{1}, \dots, \widetilde{A}^{(d)}_{c_d}]  \\
   &=  [U, U'] \left[
    \begin{array}{cc}
      \Sigma \\
      & \Sigma'
    \end{array}
   \right]
   \left[
    \begin{array}{c}
      V^\top \\
      V'^\top
    \end{array}
   \right] \approx  U \Sigma V^{\rm T}
\end{align*}
be the rank $\widehat{m}$ approximation based on SVD.
Then, the target matrix $G^{(i)}_j$ is obtained as 
\begin{equation*}
  G^{(i)}_j = \arg \min_{G \in \mathbb{R}^{\widetilde{m}_{ij} \times \widehat{m}}} \| \widetilde{A}^{(i)}_jG - Z \|_{\rm F}, \quad Z = UC,
\end{equation*}
%
where $C \in \mathbb{R}^{\widehat{m} \times \widehat{m}}$ is a nonsingular matrix.
\par
However, to construct $\widetilde{A}$, we need to share $\widetilde{A}^{(i)}_{j}$ to, e.g., the central FL server.
Sharing $\widetilde{A}^{(i)}_j$ across groups leads to an increased risk of privacy leakage and is contrary to the concept of the proposed method.
Instead, we consider computing a low-rank approximation 
\begin{align}
   \widetilde{A}^{(i)} 
   &= [\widetilde{A}^{(i)}_1, \widetilde{A}^{(i)}_2, \dots, \widetilde{A}^{(i)}_{c_i}]  \nonumber \\
   &= [U^{(i)}, U'^{(i)}] \left[
    \begin{array}{cc}
      \Sigma^{(i)} \\
      & \Sigma'^{(i)}
    \end{array}
   \right]
   \left[
    \begin{array}{c}
      (V^{(i)})^\top \\
      (V'^{(i)})^\top
    \end{array}
   \right] \approx  U^{(i)} \Sigma^{(i)} (V^{(i)})^\top
   \label{eq:svd1}
\end{align}
in intra-group DC servers and setting
\begin{equation*}
  \widetilde{B}^{(i)} = U^{(i)} C_1^{(i)},
\end{equation*}
where $C_1^{(i)} \in \mathbb{R}^{\widehat{m}_i \times \widehat{m}_i}$ is a nonsingular matrix.
Here, we have
\begin{equation*}
  \widetilde{A}^{(i)} \approx \widetilde{B}^{(i)} W^{(i)}, \quad
  W^{(i)} = (C_1^{(i)})^{-1}\Sigma^{(i)} (V^{(i)})^\top
\end{equation*}
Then, we share $\widetilde{B}^{(i)}$ to the central FL server.
\par
Let 
\begin{align}
   \widetilde{B} = [\widetilde{B}^{(1)}, \widetilde{B}^{(2)}, \dots, \widetilde{B}^{(d)}]
   = [P, P'] \left[
    \begin{array}{cc}
      D \\
      & D'
    \end{array}
   \right]
   \left[
    \begin{array}{c}
      Q^\top \\
      Q'^\top
    \end{array}
   \right] \approx  P D Q^\top
   \label{eq:svd2}
\end{align}
be the rank $\widehat{m}$ low-rank approximation based on SVD.
From the property
\begin{align*}
   \widetilde{A} 
   &= [\widetilde{A}^{(1)}, \widetilde{A}^{(2)}, \dots, \widetilde{A}^{(d)}] \\
   &\approx [\widetilde{B}^{(1)}, \widetilde{B}^{(2)}, \dots, \widetilde{B}^{(d)}]
    {\rm diag}(W^{(1)}, W^{(2)}, \dots, W^{(d)}) \\
   &\approx P D Q^\top
    {\rm diag}(W^{(1)}, W^{(2)}, \dots, W^{(d)}),
\end{align*}
we have $\mathcal{R}(U) \approx \mathcal{R}(P)$.
Using this observation, we set the matrix $G^{(i)}_j$ as
\begin{equation}
  G^{(i)}_j = \arg \min_{G \in \mathbb{R}^{\widetilde{m}_{ij} \times \widehat{m}}} \| \widetilde{A}^{(i)}_jG - Z \|_{\rm F}, \quad Z = PC_2,
  \label{eq:set_g}
\end{equation}
where $C_2 \in \mathbb{R}^{\widehat{m} \times \widehat{m}}$ is a nonsingular matrix.
\par
For example, let we split
\begin{align*}
  &(V^{(i)})^\top = [(V^{(i)}_1)^\top, (V^{(i)}_2)^\top, \dots, (V^{(i)}_{c_i})^\top], \\
  &Q^\top = [(Q^{(1)})^\top, (Q^{(2)})^\top, \dots, (Q^{(d)})^\top].
\end{align*}
Then, using random orthogonal matrix $E_1^{(i)}$ and $E_2$, we set 
\begin{align*}
  C_1^{(i)} = \Sigma (V^{(i)}_{j'_i})^\top E_1^{(i)}, \quad
  C_2 = D (Q^{(i')})^\top E_2,
\end{align*}
for randomly selected $1 \leq i' \leq d$ and $1 \leq j'_i \leq c_i$ in the numerical experiment.
%
\subsubsection*{Step 4: Construction of integrated model for collaboration representation}
The collaboration representations are given as a single dataset, that is,
\begin{equation*}
  \widehat{X}^{(i)}
  = \left[
    \begin{array}{c}
      \widehat{X}^{(i)}_1 \\
      \widehat{X}^{(i)}_2 \\
      \vdots \\
      \widehat{X}^{(i)}_{c_i}
    \end{array}
  \right] 
  = \left[
    \begin{array}{c}
      \widetilde{X}^{(i)}_1 G^{(i)}_1 \\
      \widetilde{X}^{(i)}_2 G^{(i)}_2 \\
      \vdots \\
      \widetilde{X}^{(i)}_{c_i} G^{(i)}_{c_i}
    \end{array}
  \right] 
  \in \mathbb{R}^{n_i \times \widehat{m}}, \quad
  i = 1, 2, \dots, d,
\end{equation*}
where $n_i = \sum_j n_{ij}$, and are on intra-group DC servers instead of user institutions.
Note that the intra-group DC servers, unlike the servers within the user institutions that have raw data, are capable of continuous communication with the outside world.
Therefore, integrated model 
\begin{equation*}
  h(\widehat{X}) \approx Y, \quad
  \widehat{X}
  = \left[
    \begin{array}{c}
      \widehat{X}^{(1)} \\
      \widehat{X}^{(2)} \\
      \vdots \\
      \widehat{X}^{(d)}
    \end{array}
  \right] 
\end{equation*}
can be efficiently constructed by federated learning framework with the central FL server.
Here, we note that the integrated mode $h$ is for the collaboration representations $\widehat{X}$ instead of raw data representation $X$.
\subsubsection*{Step 5: Construction of integrated model for raw data representation}
The integrated model $h$ and the matrix $G^{(i)}_j$ is returned to each $(i,j)$-th user institution from intra-group DC servers.
Then, in each user institution, integrated model for raw data representation is recovered as
\begin{equation*}
  t_j^{(i)} (X) = h(f^{(i)}_j(X)G^{(i)}_j).
\end{equation*}
\subsubsection*{Algorithm of FedDCL}
The algorithm of the proposed FedDCL is summarized in Algorithm~\ref{alg:FedDCL} and Figure~\ref{fig:algorithm}.
In the proposed FedDCL, each user institution requires only two cross-institutional communications, Steps 4 and 15 in Algorithm~\ref{alg:FedDCL}.
\begin{algorithm*}[!t]
\caption{A federated data collaboration learning (FedDCL)}
\label{alg:FedDCL}
\footnotesize
\begin{algorithmic}
  \REQUIRE Training datasets $X^{(i)}_j \in \mathbb{R}^{n_{ij} \times m}, Y^{(i)}_j \in \mathbb{R}^{n_{ij} \times \ell}$ individually
  \ENSURE Integrated model $t^{(i)}_j(X) \approx Y$ for each $i,j$
  \STATE
  \STATE
  \begin{tabular}{rclcll}
    & \multicolumn{3}{c}{ {\it Users} $(i,j)$} & \\ \cmidrule{2-4}
    1:  & \multicolumn{3}{l}{All users generate the same anchor dataset $A \in \mathbb{R}^{r \times m}$} & \\
    2:  & \multicolumn{3}{l}{Generate $f^{(i)}_j$} & \\
    3:  & \multicolumn{3}{l}{Compute $\widetilde{X}^{(i)}_j = f^{(i)}_j(X^{(i)}_j)$ and $\widetilde{A}^{(i)}_j = f^{(i)}_j(A)$} & \\
    4:  & \multicolumn{3}{l}{Share $\widetilde{X}^{(i)}_j, \widetilde{A}^{(i)}_j$, and $Y^{(i)}_j$ to {\it Intra-group DC server}} & \\
    \\
    &   & \multicolumn{3}{c}{ {\it Intra-group DC server} ($i$)}  \\ \cmidrule{3-5}
    5:  & \qquad $\searrow$ & \multicolumn{2}{l}{Obtain $\widetilde{X}^{(i)}_j, \widetilde{A}^{(i)}_j$, and $Y^{(i)}_j$ for all $j$}  \\
    6:  & & \multicolumn{3}{l}{Set $\widetilde{A}^{(i)}$ and compute a rank $\widehat{m}_i$ approximation \eqref{eq:svd1} and get $\widetilde{B}^{(i)} = U^{(i)} C_1^{(i)}$} \\
    7:  & & \multicolumn{3}{l}{Share $\widetilde{B}^{(i)}$ to {\it Central FL server}} \\
    \\
    &   & & \multicolumn{3}{c}{ {\it Central FL server}}  \\ \cmidrule{4-6}
    8:  & & \qquad $\searrow$ & \multicolumn{3}{l}{Obtain $\widetilde{B}^{(i)}$} \\
    9:  & & & \multicolumn{3}{l}{Set $\widetilde{B}$ and compute a rank $\widehat{m}$ approximation \eqref{eq:svd2} and get $Z = P C_2$} \\
    10:  & & \qquad $\swarrow$ & \multicolumn{3}{l}{Return $Z$ to {\it Intra-group DC servers}} \\
    \\
    &   & \multicolumn{3}{c}{ {\it Intra-group DC server} ($i$)}  \\ \cmidrule{3-5}
    11:  & & \multicolumn{3}{l}{Obtain $Z$} \\
    12:  & & \multicolumn{3}{l}{Compute $G^{(i)}_j$ by \eqref{eq:set_g} from $\widetilde{A}^{(i)}_j$ and $Z$ for all $j$} \\
    13:  & & \multicolumn{3}{l}{Compute $\widehat{X}^{(i)}_j = \widetilde{X}^{(i)}_j G^{(i)}_j$ for all $j$, and set $\widehat{X}^{(i)}$} \\
    14:  & & \multicolumn{3}{l}{Run federated learning with {\it Central FL server} to obtain $h(\widehat{X})$} \\
    15: & \qquad $\swarrow$ & \multicolumn{3}{l}{Return $G^{(i)}_j$ and $h(\widehat{X})$ to each user} \\
    \\
    & \multicolumn{3}{c}{ {\it User} $(i,j)$} & \\ \cmidrule{2-4}
    16: & \multicolumn{3}{l}{Obtain $G^{(i)}_j$ and $h(\widehat{X})$} \\
    17: & \multicolumn{3}{l}{Set $t_j^{(i)}(X) = h(f^{(i)}_j(X)G^{(i)}_j)$} \\
  \end{tabular}
\end{algorithmic}
\end{algorithm*}
\begin{figure*}[!t]
\centering
\includegraphics[scale=0.3, bb = 0 0 1134 567]{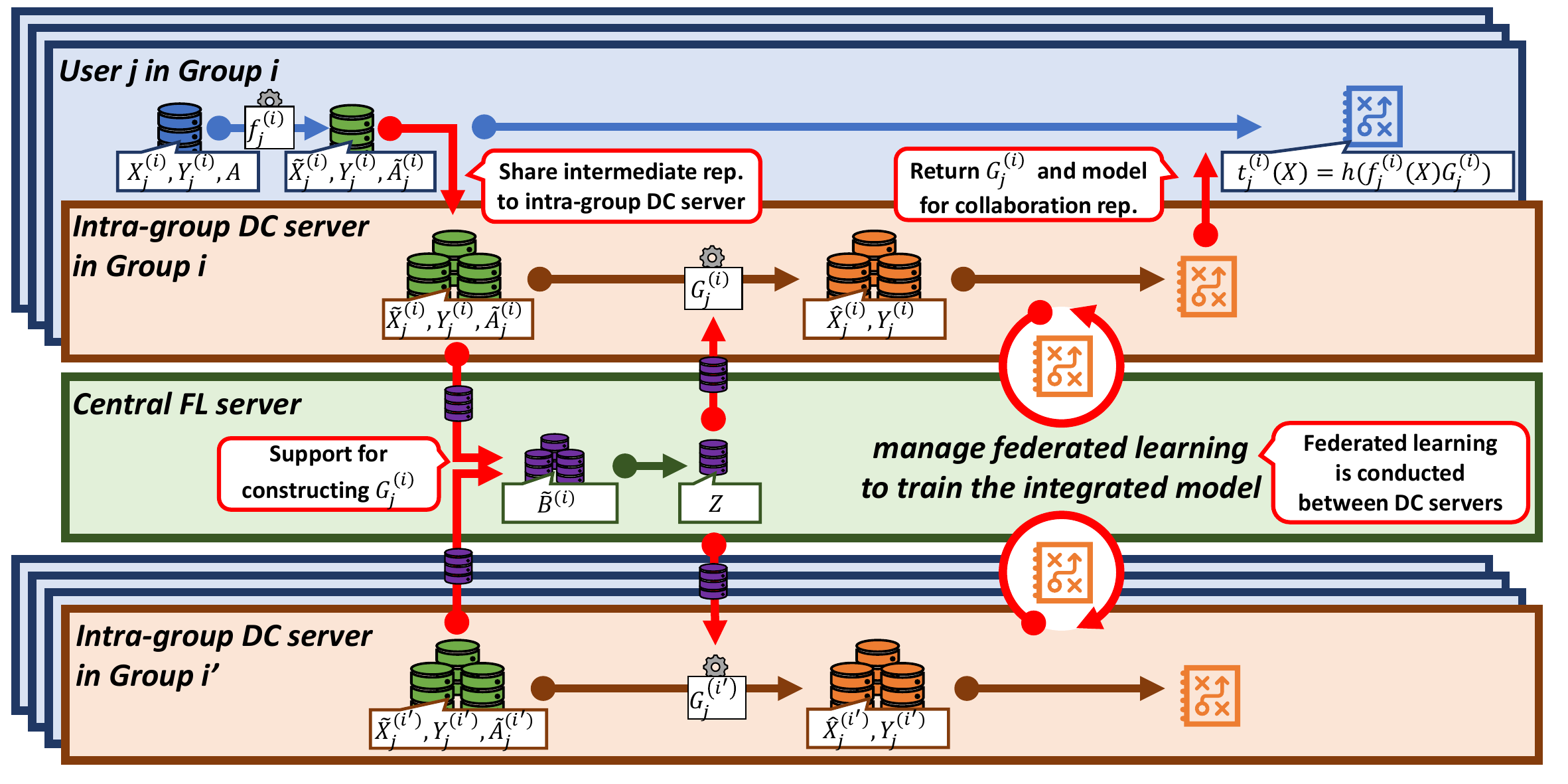}
\caption{
  Outline of the proposed FedDCL method.
}
\label{fig:algorithm}
\end{figure*}
\subsection{Discussion on correctness}
For a correctness of the proposed FedDCL, we have the following theorem.
\begin{theorem}
If the mapping functions $f^{(i)}_j$ are linear, that is, $f^{(i)}_j(X^{(i)}_j) = X^{(i)}_jF^{(i)}_j$ with $F^{(i)}_j \in \mathbb{R}^{m \times \widetilde{m}}$ and the matrices $F^{(i)}_j$ have the same range 
\begin{equation}
  \mathcal{F} = \mathcal{R}(F^{(1)}_1) = \dots = \mathcal{R}(F^{(c)}_{d_c}), \quad {\rm rank}(AF^{(i)}_j) = \widetilde{m}.
  \label{eq:cond}
\end{equation}
%
Then, for the collaboration representations $\widehat{X}$ of the FedDCL, there exist the dimensionality reduction $F$ such that
\begin{equation*}
  \widehat{X} = XF, \quad \mathcal{R}(F) = \mathcal{F}.
\end{equation*}
\end{theorem}
\begin{proof}
  If $F^{(i)}_j$ have the same range \eqref{eq:cond}, then $F^{(i)}_j = F^{(1)}_1 E^{(i)}_j$ with $E^{(i)}_j \in \mathbb{R}^{\widetilde{m} \times \widetilde{m}}$.
  Therefore, we also have $\widetilde{A}^{(i)}_j = AF^{(1)}_1 E^{(i)}_j$.
  In this case, since $\Sigma'^{(i)} = O$ in \eqref{eq:svd1} and $D' = O$ in \eqref{eq:svd2},
  \begin{equation*}
    \widetilde{A} 
    = P D Q^\top
    {\rm diag}(W^{(1)}, W^{(2)}, \dots, W^{(d)}).
  \end{equation*}
  This leads to
  \begin{equation*}
    \min_{G \in \mathbb{R}^{\widetilde{m}_{ij} \times \widehat{m}}} \| \widetilde{A}^{(i)}_jG - P C_2 \|_{\rm F} = 0
  \end{equation*}
  for all $i$ and $j$.
  From $\widetilde{A}^{(i)}_j = A F_j^{(i)}$, we have $F_j^{(i)} G_j^{(i)} = F_{j'}^{(i')} G_{j'}^{(i')}$ and define $F = F_j^{(i)} G_j^{(i)}$.
  Therefore, we have
  \begin{equation*}
      \widehat{X}^{(i)}_j = X_j^{(i)} F_j^{(i)} G_j^{(i)} = X_j^{(i)} F
  \end{equation*}
  for all $i$ and $j$ that proves Theorem~1.
\end{proof}
\par
Theorem~1 implies that, under the conditions \eqref{eq:cond}, FedDCL is equivalent to federated learning for dimensionality-reduced data constructed by the same mapping function.
\subsection{Discussion on privacy}
Here, we discuss data privacy of FedDCL with respect to information leakage from data held by the intra-group DC server and the central FL server.
\par
First, we consider information leakage from data held by the intra-group DC server.
Each intra-group DC server hold intermediate and collaboration representations of user institutions in the group.
Here, the private data $X^{(i)}_j$ is protected by the following double privacy layer:
\begin{enumerate}[{Layer} 1]
  \item No one can possess private data $X^{(i)}_j$ because $f^{(i)}_j$ is private under the protocol;
  \item Even if $f^{(i)}_j$ is stolen, the private data $X^{(i)}_j$ is still protected regarding $\varepsilon$-DR privacy \cite{nguyen2020autogan} because $f^{(i)}_j$ is a dimensionality reduction function,
\end{enumerate}
as a manner identical to that of the conventional data collaboration analysis shown in \cite{imakura2021accuracy}.
While conventional data collaboration analysis centralizes intermediate representations on a single central server, FedDCL centralizes them to the intra-group DC server in each group.
In this sense, FedDCL reduces the risk of a single point of failure compared to conventional data collaboration analysis.
\par
Second, we consider information leakage from data held by the central FL server.
The information held by the central FL server is basically equal to that for conventional federated learning.
As well as conventional federated learning, there is a possibility of information leakage from, e.g., gradient information.
However, while the information leakage risk of conventional federated learning is for raw data, the information leakage risk of FedDCL is for collaboration representations.
In this sense, FedDCL reduces the risk of information leakage for the raw data compared to conventional federated learning.
\section{Numerical evaluations}
\subsection{Experimental conditions}
This section provides a comparison of the proposed {\bf FedDCL} (Algorithm~\ref{alg:FedDCL}) with the centralized analysis ({\bf Centralized}), which shares the original dataset, the local analysis ({\bf Local}), which uses only one local dataset, federated learning ({\bf FedAvg}), and data collaboration analysis ({\bf DC}).
Note that the proposed FedDCL is a framework that can be easily combined with the latest federated learning mechanisms.
Therefore, in this paper, we just evaluate the performance with a simple {\bf FedAvg}.
\par
For the machine learning model, we use fully connected neural network.
For all methods, we set batch size as 32.
For {\bf Centralized}, {\bf Local}, and {\bf DC}, the number of epoch is set as 40.
For {\bf FedAvg} and {\bf FedDCL}, the number of epochs in each round is set as 4 and the number of rounds is set as 20 (total number of epochs is 80).
This is based on the fact that the convergence of {\bf FedAvg} is generally lower than {\bf Centralized}.
For {\bf DC} and {\bf FedDCL}, we used PCA with random orthogonal mapping for constructing intermediate representations.
Anchor dataset $A$ was constructed as a random matrix in the range of the corresponding feature, as in \cite{imakura2020data,imakura2021collaborative}.
We set $\widehat{m} = \widetilde{m}_{ij}$ as dimensionality of collaboration representations and $r = 2000$ as the number of anchor data.
\par
All random values were generated by Mersenne Twister.
All the numerical experiments were conducted on Windows 11, 13th Gen Intel(R) Core(TM) i7-1370P @ 1.90 GHz, 64GB RAM using MATLAB2024a.
\subsection{Experiment I: Proof-of-concept}
\par
As a proof-of-concept, we evaluate the efficiency of the proposed FedDCL on {\sf BatterySmall} dataset obtained from the MATALB Statistics and Machine Learning Toolbox.
{\sf BatterySmall} is a dataset of lithium-ion battery sensor data (voltage (V), electric current (I), temperature (Temp), average voltage (V\_avg), average electric current (I\_avg)) and data on the battery's state of charge (SOC).
This is a subset of the data in \cite{kollmeyer2020lg}.
We set the dataset up as regression problem.
\par
We consider the situation where the dataset is held in four user institutions which also organize into two groups, that is, $c_i = d = 2$.
Each user institution has 100 samples, that is, $n_{ij} = 100$.
A part of raw data of users are shown in Table~\ref{table:raw_data}.
We set dimensionality of intermediate representations as $\widehat{m} = \widetilde{m}_{ij} = 4$.
We also set the layers of neural network as [5--20--1] for {\bf Centralized}, {\bf Local}, and {\bf FedAvg} and [4--20--1] for {\bf DC} and {\bf FedDCL}.
Note that, for {\bf DC} and {\bf FedDCL}, neural network is applied to the collaboration representation $\widehat{X}$ with dimensionality $\widehat{m} = 4$.
We set the number of test samples is $1000$.
\begin{table*}[!t]
\centering
\caption{A part of raw dataset held by users.}
\label{table:raw_data}
\begin{tabular}{rrrrrrrrrrr}
\toprule
\multicolumn{5}{c}{User (1,1) $X^{(1)}_1$ } & & \multicolumn{5}{c}{User (1,2) ${X}^{(1)}_2$} \\ \cmidrule{1-5} \cmidrule{7-11}
\multicolumn{1}{c}{V} & \multicolumn{1}{c}{I} & \multicolumn{1}{c}{Temp} & \multicolumn{1}{c}{V\_avg} & \multicolumn{1}{c}{I\_avg} & & \multicolumn{1}{c}{V} & \multicolumn{1}{c}{I} & \multicolumn{1}{c}{Temp} & \multicolumn{1}{c}{V\_avg} & \multicolumn{1}{c}{I\_avg} \\ \cmidrule{1-5} \cmidrule{7-11}
$0.978$ & $0.754$ & $0.921$ & $0.978$ &$0.755$ & \phantom{0} &$0.330$ &$0.751$ &$0.916$ &$0.329$ &$0.751$ \\
$0.978$ & $0.756$ & $0.918$ & $0.978$ &$0.759$ & &$0.978$ &$0.839$ &$0.947$ &$0.951$ &$0.868$ \\
$0.386$ & $0.751$ & $0.492$ & $0.385$ &$0.751$ & &$0.493$ &$0.751$ &$0.924$ &$0.493$ &$0.751$ \\
$0.978$ & $0.759$ & $0.921$ & $0.978$ &$0.765$ & &$0.411$ &$0.751$ &$0.912$ &$0.410$ &$0.751$ \\ \midrule

\\ \midrule
\multicolumn{5}{c}{User (2,1) $X^{(2)}_1$ } & & \multicolumn{5}{c}{User (2,2) ${X}^{(2)}_2$} \\ \cmidrule{1-5} \cmidrule{7-11}
\multicolumn{1}{c}{V} & \multicolumn{1}{c}{I} & \multicolumn{1}{c}{Temp} & \multicolumn{1}{c}{V\_avg} & \multicolumn{1}{c}{I\_avg} & & \multicolumn{1}{c}{V} & \multicolumn{1}{c}{I} & \multicolumn{1}{c}{Temp} & \multicolumn{1}{c}{V\_avg} & \multicolumn{1}{c}{I\_avg} \\ \cmidrule{1-5} \cmidrule{7-11}
$0.495$ &$0.751$ &$0.919$ &$0.495$ &$0.751$ & &$0.655$ &$0.656$ &$0.040$ &$0.700$ &$0.669$ \\ 
$0.299$ &$0.751$ &$0.918$ &$0.279$ &$0.751$ & &$0.600$ &$0.679$ &$0.026$ &$0.605$ &$0.705$ \\
$0.427$ &$0.663$ &$0.040$ &$0.570$ &$0.707$ & &$0.312$ &$0.751$ &$0.915$ &$0.312$ &$0.751$ \\
$0.487$ &$0.769$ &$0.955$ &$0.433$ &$0.643$ & &$0.314$ &$0.751$ &$0.918$ &$0.313$ &$0.751$ \\

\bottomrule
\end{tabular}
\end{table*}
\begin{table*}[p]
\centering
\caption{A part of intermediate and collaboration representations for users.}
\label{table:data}
\begin{tabular}{rrrrrrrrr}
\toprule
\multicolumn{4}{c}{User (1,1) $\widetilde{X}^{(1)}_1$ } & & \multicolumn{4}{c}{User (1,1) $\widehat{X}^{(1)}_1$} \\ \cmidrule{1-4} \cmidrule{6-9}
$0.863$ &$0.724$ &$-0.978$	&$1.286$ & &$0.896$ &$-0.282$	&$0.277$	&$-1.709$ \\
$0.862$ &$0.723$ &$-0.977$	&$1.290$ & &$0.893$ &$-0.283$	&$0.280$	&$-1.711$ \\
$0.237$ &$0.579$ &$-0.325$	&$1.073$ & &$0.359$ &$-0.060$	&$0.627$	&$-1.058$ \\
$0.862$ &$0.726$ &$-0.975$	&$1.297$ & &$0.894$ &$-0.283$	&$0.286$	&$-1.715$ \\ \midrule

\\ \midrule
\multicolumn{4}{c}{User (1,2) $\widetilde{X}^{(1)}_2$ } & & \multicolumn{4}{c}{User (1,2) $\widehat{X}^{(1)}_2$} \\ \cmidrule{1-4} \cmidrule{6-9}
$0.366$ & $-1.069$ & $-0.486$ &	$0.809$ && $0.759$ &$-0.020$ &	$0.771$ &$-0.996$ \\
$0.040$ & $-1.072$ & $-0.553$ &	$1.656$ && $0.892$ &$-0.278$ &	$0.405$ &$-1.777$ \\
$0.277$ & $-1.054$ & $-0.470$ &	$1.021$ && $0.800$ &$-0.091$ &	$0.648$ &$-1.174$ \\
$0.319$ & $-1.056$ & $-0.481$ &	$0.912$ && $0.772$ &$-0.055$ &	$0.709$ &$-1.084$ \\ \midrule

\\ \midrule
\multicolumn{4}{c}{User (2,1) $\widetilde{X}^{(2)}_1$ } & & \multicolumn{4}{c}{User (2,1) $\widehat{X}^{(2)}_1$} \\ \cmidrule{1-4} \cmidrule{6-9}
$-0.587$ &$-1.257$ &$-0.117$ &$-0.716$ &&$ 0.795$ &$-0.082$ &$0.645$ &$-1.179$ \\
$-0.425$ &$-1.297$ &$-0.176$ &$-0.485$ &&$ 0.752$ &$ 0.009$ &$0.801$ &$-0.955$ \\
$-0.994$ &$-0.589$ &$-0.035$ &$-0.300$ &&$-0.037$ &$-0.214$ &$0.374$ &$-1.105$ \\
$-0.491$ &$-1.228$ &$-0.197$ &$-0.733$ &&$ 0.833$ &$ 0.014$ &$0.615$ &$-1.120$ \\ \midrule

\\ \midrule
\multicolumn{4}{c}{User (2,2) $\widetilde{X}^{(2)}_2$ } & & \multicolumn{4}{c}{User (2,2) $\widehat{X}^{(2)}_2$} \\ \cmidrule{1-4} \cmidrule{6-9}
$-0.100$ &$-1.321$ &$ 0.145$ &$ 0.012$ &&$ 0.005$ &$-0.229$ &$0.203$ &$-1.297$ \\
$-0.001$ &$-1.277$ &$ 0.188$ &$ 0.006$ &&$-0.033$ &$-0.188$ &$0.290$ &$-1.244$ \\
$ 0.116$ &$-1.021$ &$ 0.586$ &$ 0.860$ &&$ 0.755$ &$-0.007$ &$0.783$ &$-0.978$ \\
$ 0.114$ &$-1.022$ &$ 0.587$ &$ 0.862$ &&$ 0.758$ &$-0.008$ &$0.782$ &$-0.979$ \\

\bottomrule
\end{tabular}
\end{table*}
\par
We first demonstrate the intermediate and collaboration representations.
A part of intermediate and collaboration representations are shown in Table~\ref{table:data}.
These results indicate that the intermediate and collaboration representations do not directly approximate the features of the raw data.
Focusing on the range of values in each column, we see that the intermediate representation varies widely from institution to institution, while the collaboration representation is generally consistent.
It is not possible to recover the raw data only from the intermediate and collaboration representations.
\par
Next, we evaluate recognition performance.
The convergence history of root mean squared error (RMSE) of FedDCL and other methods are shown in Figure~\ref{fig:battery_history}.
For {\bf FedAvg} and {\bf FedDCL}, the convergence history was plotted for each round, that is, every 4 epochs.
The experimental results show that {\bf FedDCL} has a higher convergence than {\bf FedAvg}.
This may be partly due to the fact that {\bf FedDCL} using dimensionality-reduced intermediate representation has fewer weight parameters than {\bf FedAvg}.
\par
In total, experimental results demonstrate that {\bf FedDCL} functions correctly as a privacy-preserving integrated analysis.
\begin{figure}[!t]
\center
\includegraphics[bb =  50 250 545 600, scale=0.5]{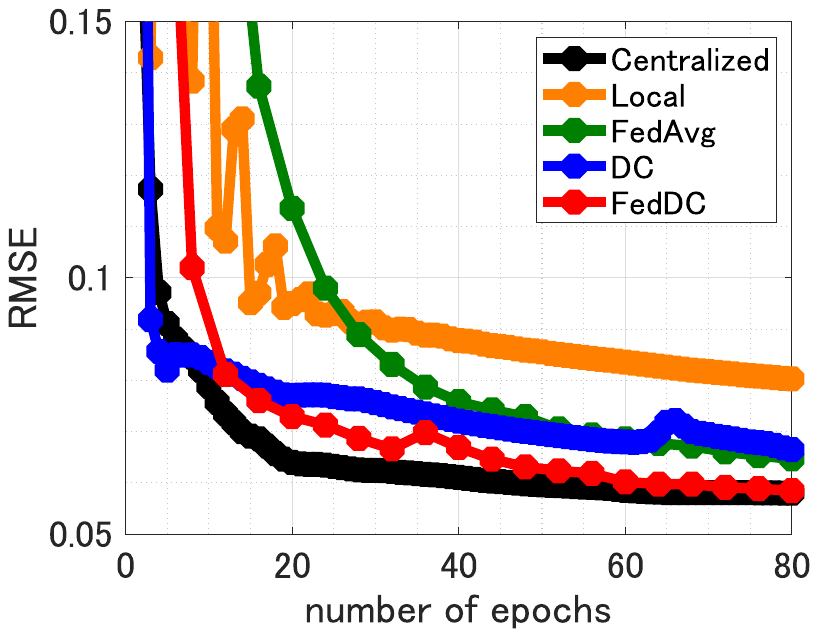}
\caption{Convergence history for {\sf BatterySmall}.
{\it Remark: {\bf FedDCL} has a higher convergence than {\bf FedAvg} and {\bf DC}.}}
\label{fig:battery_history}
\end{figure}
\subsection{Experiment II: prediction performance for six datasets}
We evaluate the prediction performance for six datasets.
\begin{itemize}
  \item {\sf BatterySmall} used in Experiment I.
  \item 
  {\sf CreditRating\_Historical} is a dataset contains five financial ratios, i.e., Working capital / Total Assets (WC\_TA), Retained Earnings / Total Assets (RE\_TA), Earnings Before Interests and Taxes / Total Assets (EBIT\_TA), Market Value of Equity / Book Value of Total Debt (MVE\_BVTD), and Sales / Total Assets (S\_TA), and industry sector labels from 1 to 12 for 3932 customers, obtained from the MATLAB Statistics and Machine Learning Toolbox.
  The dataset also includes credit ratings from ``AAA'' to ``CCC'' for all customers.
  We quantified each ``AAA'' to ``CCC'' from 6 to 0 and set up as regression problem.
  \item {\sf eICU} is a large critical care database gathered from multiple hospitals in the U.S. obtained from the eICU Collaborative Research Database \cite{pollard2018eicu}.
  We selected $24$ features: gender, age, apacheApsVar, intubated, vent, eyes, motor, verbal, urine, wbc, temperature, respiratoryrate, sodium, heartrate, meanbp, ph, hematocrit, creatinine, albumin, pao2, pco2, bun, glucose, bilirubin, and fio2.
  We set up as a regression problem for the number of days in the unit.
  \item {\sf HumanActivity} is a dataset for five human activities: sitting, standing, walking, running, and dancing, obtained from the MATLAB Statistics and Machine Learning Toolbox.
  We set up as a five class classification problem.
  \item {\sf MNIST} is a handwritten digit database \cite{lecun1998gradient}.
  It has $28 \times 28$ grayscale images with a label from 10 classes.
  We set up as a 10 class classification problem.
  \item {\sf Fashion-MNIST} is a dataset of Zalando's article images \cite{xiao2017fashion}.
  It has $28 \times 28$ grayscale images with a label from 10 classes.
  We set up as a 10 class classification problem.
\end{itemize}
We consider the situation where each dataset is held in 20 user institutions which also organize into 5 groups, that is, $d = 5$ and $c_i = 4$.
Other parameters are shown in Table~\ref{table:parameters}.
We evaluate the prediction performance: root mean squared error (RMSE) for {\sf BatterySmall}, {\sf CreditRating\_Historical}, and {\sf eICU} and Accuracy for {\sf HumanActivity}, {\sf MNIST}, and {\sf Fashion-MNIST}.
\begin{table}[t]
  \caption{
    Parameters for Experiment I.
  }
  \label{table:parameters}
\centering
\begin{tabular}{ccccc} \toprule
Dataset & $n_{ij}$ & $m$ & $\widehat{m}=\widetilde{m}_{ij}$ & network layers \\ \midrule
{\sf BatterySmall} & 100 & 5 & 4 & [$\{m,\widehat{m}\}$--20--1]  \\
{\sf CreditRating\_Historical} & 100 & 17 & 15 & [$\{m,\widehat{m}\}$-50-1]  \\
{\sf eICU} & 100 & 24 & 15 & [$\{m,\widehat{m}\}$--10--1] \\
{\sf HumanActivity} & 100 & 60 & 50 & [$\{m,\widehat{m}\}$--80--5] \\
{\sf MNIST} & 100 & 784 & 50 & [$\{m,\widehat{m}\}$--500--100--10] \\
{\sf Fashion-MNIST} & 1000 & 784 & 50 & [$\{m,\widehat{m}\}$--500--100--10] \\
\bottomrule
\end{tabular}
\end{table}
\begin{figure}[!t]
\center
\subfloat[Batterysmall (RMSE)]{
\includegraphics[bb =  50 250 545 600, scale=0.35]{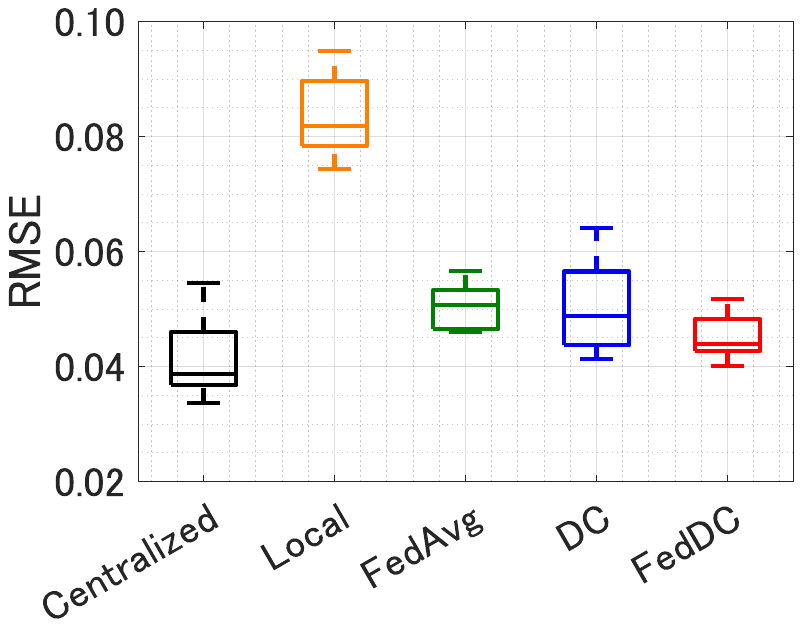}
}
\subfloat[CreditRating\_Historical (RMSE)]{
\includegraphics[bb =  50 250 545 600, scale=0.35]{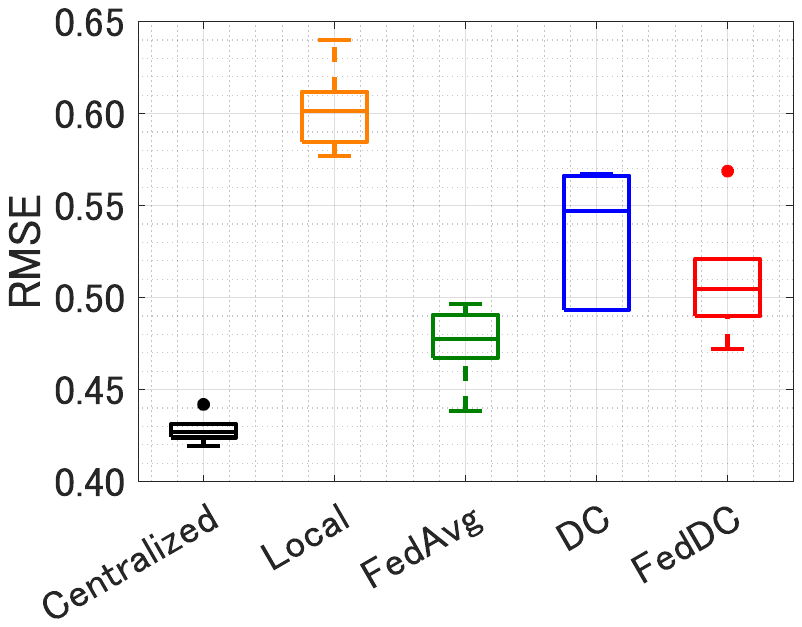}
} \\
\subfloat[eICU (RMSE)]{
\includegraphics[bb =  50 250 545 600, scale=0.35]{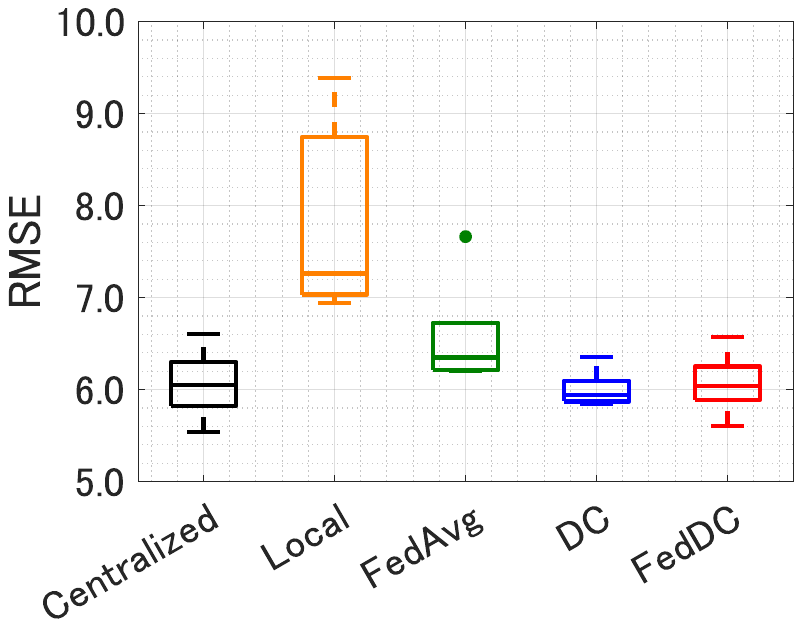}
}
\subfloat[HumanActivity (Accuracy)]{
\includegraphics[bb =  50 250 545 600, scale=0.35]{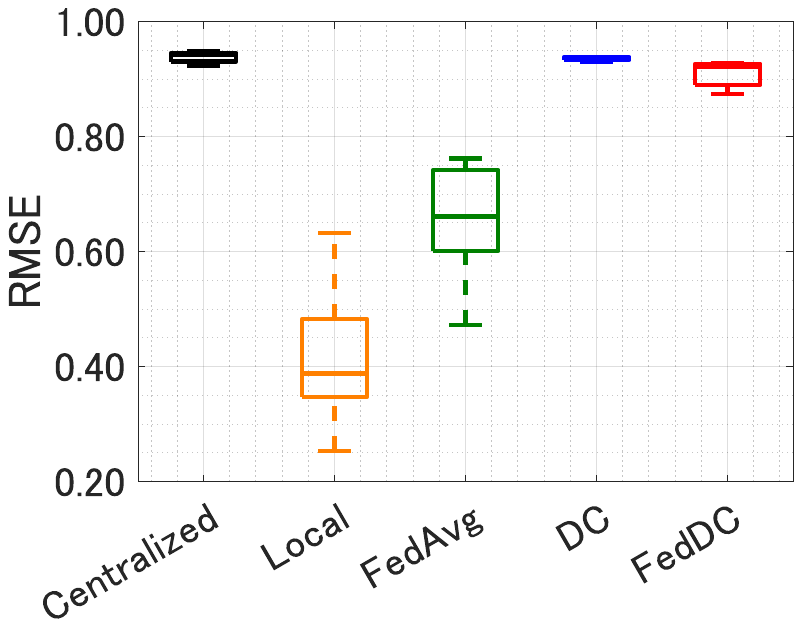}
} \\
\subfloat[MNIST (Accuracy)]{
\includegraphics[bb =  50 250 545 600, scale=0.35]{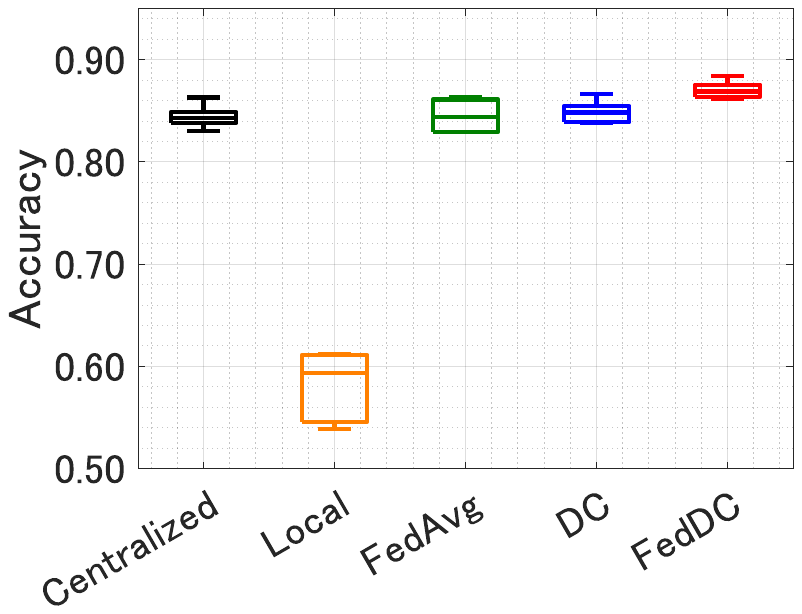}
}
\subfloat[Fashion-MNIST (Accuracy)]{
\includegraphics[bb =  50 250 545 600, scale=0.35]{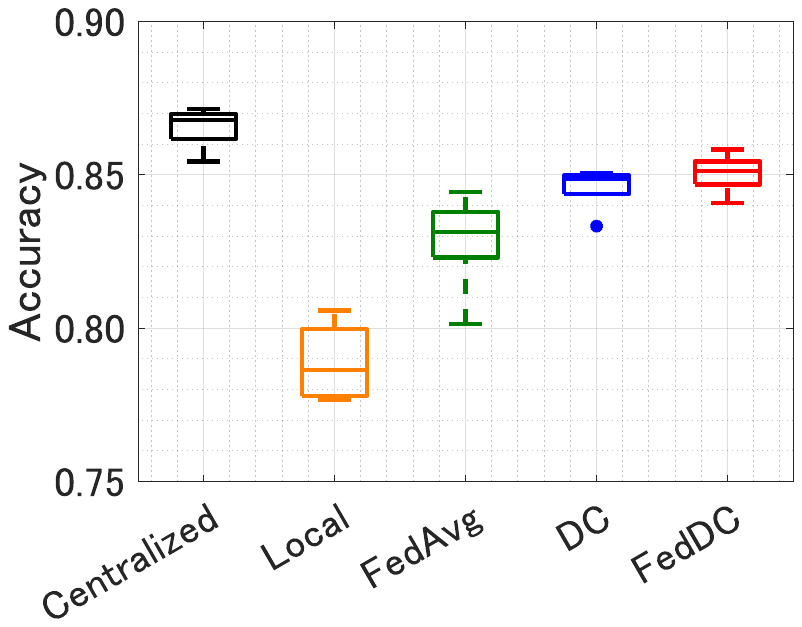}
}
\caption{Prediction performance. Note that lower RMSE and higher Accuracy mean better recognition performance.
{\it Remark: {\bf FedDCL} demonstrates very high recognition performance compared to {\bf Local} and comparable to {\bf FedAvg} and {\bf DC}.}}
\label{fig:ex1}
\end{figure}
\par
Numerical results are shown in Figure~\ref{fig:ex1}.
Note that lower RMSE and higher Accuracy mean better recognition performance.
Experimental results demonstrate that {\bf FedDCL} has very high recognition performance compared to {\bf Local} and comparable to {\bf FedAvg} and {\bf DC}.
\subsection{Experiment III: Performance improvement for increasing number of groups}
We evaluate the performance improvement when the number of groups is increased to $d=1, 2, \dots, 10$ with $c_i = 4$ for {\sf MNIST}.
Other parameters were set as in Experiment II.
\par
Numerical results are shown in Figure~\ref{fig:ex2}.
The results show that the accuracy of {\bf FedDCL} increases with increasing the number of groups as well as {\bf Centralized}, {\bf DC}, and {\bf DC}.
In addition, {\bf FedDCL} showed higher recognition performance than {\bf Centralized}.
This may be due to the higher number of epochs for federated learning.
Since {\bf FedAvg} generally has lower convergence than {\bf Centralized}, we set large total number of epochs for {\bf FedAvg} and {\bf FedDCL}.
More detailed analysis is a subject for future work.
\begin{figure}[!t]
\center
\includegraphics[bb =  50 250 545 600, scale=0.5]{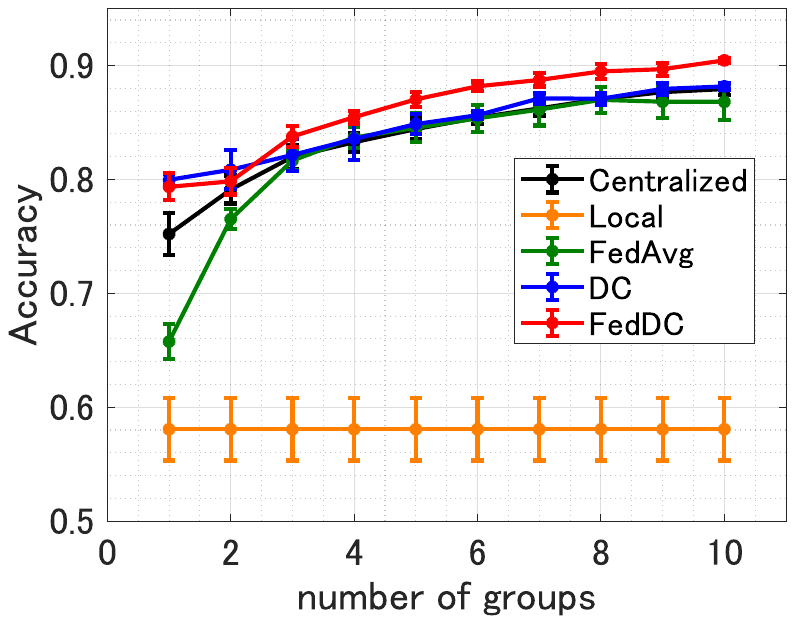}
\caption{Prediction performance vs number of groups.
{\it Remark: the accuracy of {\bf FedDCL} increases with increasing the number of groups as well as {\bf Centralized}, {\bf DC}, and {\bf DC}.}}
\label{fig:ex2}
\end{figure}
\section{Conclusions}
In recent years, there has been a growing need for privacy-preserving integrated analysis for medical data held by multiple institutions.
Medical data may be stored on a server isolated from external networks, making continuous communication with the outside extremely difficult.
Therefore, developments of technologies for privacy-preserving integrated analysis without any iterative communications by user institutions are essential.
\par
In this study, we propose the FedDCL framework, which solves such communication issues by combining federated learning with the data collaboration analysis; see Figure~\ref{fig:abstract}.
FedDCL is the first method to combine federated learning and data collaboration analysis.
FedDCL has experimentally shown comparable analysis performance to existing federated learning and data collaboration analysis.
\par
FedDCL is a framework that can be easily combined with the latest federated learning mechanisms and deployed on a variety of data and tasks.
Therefore, FedDCL could become a breakthrough technology for future privacy-preserving analyses on multiple institutions in situations where continuous communication with the outside world is extremely difficult.
\par
Performance evaluations for parameter dependency and for non-IID distributed data have been done separately for federated learning and data collaboration analysis.
A similar evaluation for FedDCL is a future task.
In the future, we will develop the method and software.

\section*{Acknowledgements}
This work was supported in part by the Japan Society for the Promotion of Science (JSPS), Grants-in-Aid for Scientific Research (Nos. JP22K19767, JP23H00462, JP23K21673, JP23K22166, JP23K28101, JP24K00535).

\bibliography{mybibfile}
\bibliographystyle{elsart-num-sort}

\end{document}